\documentclass{article}
\usepackage{ijcai22}

\usepackage{times}

\usepackage{soul}
\usepackage{url}
\usepackage[hidelinks]{hyperref}
\usepackage[utf8]{inputenc}
\usepackage[small]{caption}
\usepackage{graphicx}
\usepackage{amsmath}
\usepackage{booktabs}
\usepackage{amsthm}

\usepackage{xcolor}
\usepackage{amssymb}
\usepackage{amsfonts}
\usepackage{bm}
\usepackage{multicol}
\usepackage{multirow}
\usepackage{xifthen}
\usepackage{siunitx}
\usepackage{xspace}
\newcommand{\nc}{\newcommand}

\nc{\calDelta}{\mbox{$\cal{\Delta}$}}

\nc{\calA}{{\ensuremath{\cal{A}}}} \nc{\calB}{{\ensuremath{\cal{B}}}}
\nc{\calC}{{\ensuremath{\cal{C}}}} \nc{\calD}{{\ensuremath{\cal{D}}}}
\nc{\calE}{{\ensuremath{\cal{E}}}} \nc{\calF}{{\ensuremath{\cal{F}}}}
\nc{\calG}{{\ensuremath{\cal{G}}}} \nc{\calH}{{\ensuremath{\cal{H}}}}
\nc{\calI}{{\ensuremath{\cal{I}}}} \nc{\calJ}{{\ensuremath{\cal{J}}}}
\nc{\calK}{{\ensuremath{\cal{K}}}} \nc{\calL}{{\ensuremath{\cal{L}}}}
\nc{\calM}{{\ensuremath{\cal{M}}}} \nc{\calN}{{\ensuremath{\cal{N}}}}
\nc{\calO}{{\ensuremath{\cal{O}}}} \nc{\calP}{{\ensuremath{\cal{P}}}}
\nc{\calQ}{{\ensuremath{\cal{Q}}}} \nc{\calR}{{\ensuremath{\cal{R}}}}
\nc{\calS}{{\ensuremath{\cal{S}}}} \nc{\calT}{{\ensuremath{\cal{T}}}}
\nc{\calU}{{\ensuremath{\cal{U}}}} \nc{\calV}{{\ensuremath{\cal{V}}}}
\nc{\calW}{{\ensuremath{\cal{W}}}} \nc{\calX}{{\ensuremath{\cal{X}}}}
\nc{\calY}{{\ensuremath{\cal{Y}}}} \nc{\calZ}{{\ensuremath{\cal{Z}}}}
\nc{\calz}{{\ensuremath{\cal{z}}}}

\nc{\bmnull}{{\ensuremath{\boldsymbol{0}}}} 
\nc{\bma}{{\ensuremath{\boldsymbol{a}}}} 
\nc{\bmb}{{\ensuremath{\boldsymbol{b}}}}
\nc{\bmc}{{\ensuremath{\boldsymbol{c}}}} \nc{\bmd}{{\ensuremath{\boldsymbol{d}}}}
\nc{\bme}{{\ensuremath{\boldsymbol{e}}}} \nc{\bmf}{{\ensuremath{\boldsymbol{f}}}}
\nc{\bmg}{{\ensuremath{\boldsymbol{g}}}} \nc{\bmh}{{\ensuremath{\boldsymbol{h}}}}
\nc{\bmi}{{\ensuremath{\boldsymbol{i}}}} \nc{\bmj}{{\ensuremath{\boldsymbol{j}}}}
\nc{\bmk}{{\ensuremath{\boldsymbol{k}}}} \nc{\bml}{{\ensuremath{\boldsymbol{l}}}}
\nc{\bmm}{{\ensuremath{\boldsymbol{m}}}} \nc{\bmn}{{\ensuremath{\boldsymbol{n}}}}
\nc{\bmo}{{\ensuremath{\boldsymbol{o}}}} \nc{\bmp}{{\ensuremath{\boldsymbol{p}}}}
\nc{\bmq}{{\ensuremath{\boldsymbol{q}}}} \nc{\bmr}{{\ensuremath{\boldsymbol{r}}}}
\nc{\bms}{{\ensuremath{\boldsymbol{s}}}} \nc{\bmt}{{\ensuremath{\boldsymbol{t}}}}
\nc{\bmu}{{\ensuremath{\boldsymbol{u}}}} \nc{\bmv}{{\ensuremath{\boldsymbol{v}}}}
\nc{\bmw}{{\ensuremath{\boldsymbol{w}}}} \nc{\bmx}{{\ensuremath{\boldsymbol{x}}}}
\nc{\bmy}{{\ensuremath{\boldsymbol{y}}}} \nc{\bmz}{{\ensuremath{\boldsymbol{z}}}}

\nc{\bmUpsilon}{{\mbox{\bm $\Upsilon$}}}
\nc{\bmupsilon}{{\mbox{\bm $\upsilon$}}}
\nc{\bmalpha}{{\mbox{\bm $\alpha$}}}
\nc{\bmbeta}{{\mbox{\bm $\beta$}}}
\nc{\bmgamma}{{\mbox{\bm $\gamma$}}}
\nc{\bmGamma}{{\mbox{\bm $\Gamma$}}}
\nc{\bmdelta}{{\mbox{\bm $\delta$}}}
\nc{\bmDelta}{{\mbox{\bm $\Delta$}}}
\nc{\bmeps}{{\mbox{\bm $\epsilon$}}}
\nc{\bmepsilon}{\bmeps}
\nc{\bmphi}{{\mbox{\bm $\phi$}}}
\nc{\bmPhi}{{\mbox{\bm $\Phi$}}}
\nc{\bmlambda}{{\mbox{\bm $\lambda$}}}
\nc{\bmLambda}{{\mbox{\bm $\Lambda$}}}
\nc{\bmmu}{{\mbox{\bm $\mu$}}}
\nc{\bmnu}{{\mbox{\bm $\nu$}}}
\nc{\bmpi}{{\mbox{\bm $\pi$}}}
\nc{\bmrho}{{\mbox{\bm$\rho$}}}
\nc{\bmpsi}{{\mbox{\bm $\psi$}}}
\nc{\bmPsi}{{\mbox{\bm $\Psi$}}}
\nc{\bmsigma}{{\mbox{\bm $\sigma$}}}
\nc{\bmSigma}{{\mbox{\bm $\Sigma$}}}
\nc{\bmtheta}{{\ensuremath{\boldsymbol{\theta}}}}
\nc{\bmTheta}{{\mbox{\bm $\Theta$}}}
\nc{\bmeta}{{\mbox{\bm $\eta$}}}
\nc{\bmzeta}{{\mbox{\bm $\zeta$}}}
\nc{\bmxi}{{\mbox{\bm $\xi$}}} \nc{\bmXi}{{\mbox{\bm $\Xi$}}}
\nc{\bmnabla}{{\mbox{\bm $\nabla$}}}
\nc{\bmomega}{{\mbox{\bm $\omega$}}}
\nc{\bmOmega}{{\mbox{\bm $\Omega$}}}

\nc{\given}{|}
\newcommand{\T}{^{\mbox{\tiny{T}}}}

\newcommand{\FW}{IGN\xspace}
\newcommand{\FWs}{IGNs\xspace}

\newtheorem{prop}{Proposition}

\definecolor{dimgray}{rgb}{0.0, 0.0, 0.0}
\definecolor{darkgreen}{rgb}{0.0, 0.545, 0.0}
\newcommand{\alessandro}[1]{%
{\color{darkgreen}#1}%
}

\newcommand{\frmt}[1]{%
  \num[round-mode=places,round-precision=2,detect-weight=true,detect-family=true,number-unit-product =0,tight-spacing=true]{#1}%
}

\newcommand{\hide}[1]{}

\newcommand{\tr}[3][]{%
\ifthenelse{\equal{#1}{}}{${\text{\frmt{#2}}}_{\textcolor{dimgray}{\pm\frmt{#3}}}$}{\textbf{\frmt{#2}}$_{\textcolor{dimgray}{\pm\text{\textbf{\frmt{#3}}} }}$}%
}

\title{Inducing Gaussian Process Networks}

\author{
Alessandro Tibo\footnote{Contact Author}\And
Thomas Dyhre Nielsen
\affiliations
Computer Science Department, Aalborg University\\
\emails
\{alessandro, tdn\}@cs.aau.dk
}

\begin{document}

\maketitle
\begin{abstract}
    Gaussian processes (GPs) are powerful but computationally expensive machine learning models, requiring an estimate of the kernel covariance matrix for every prediction. In large and complex domains, such as graphs, sets, or images, the choice of suitable kernel can also be non-trivial to determine, providing an additional obstacle to the learning task. Over the last decade, these challenges have resulted in significant advances being made in terms of scalability and expressivity, exemplified by, e.g., the use of inducing points and neural network kernel approximations.
    In this paper, we propose inducing Gaussian process networks (\FW), a simple framework for simultaneously learning the feature space as well as the inducing points. The inducing points, in particular, are learned directly in the feature space, enabling a seamless representation of complex structured domains while also facilitating scalable gradient-based learning methods. 
    We consider both regression and (binary) classification tasks and report on experimental results for real-world data sets showing that \FWs provide significant advances over state-of-the-art methods. We also demonstrate how \FWs can be used to effectively model complex domains using neural network architectures.

\hide{
    Gaussian processes are extremely powerful machine learning models, which along with predictions provide an estimate of the uncertainty associated with the predictions.
    However, being non-parametric models, Gaussian processes are computational expensive as they required an estimate of the kernel covariance matrix for every prediction. Furthermore, the choice of the kernel is not trivial in certain complex domains, such as graphs, sets, or images. To address these challenges, significant progress have been made in reducing the computational cost by, e.g., using inducing points and alleviating the kernel specification by neural network approximations. However, for complex domains, these tasks are still challenging. 
    To alleviate the problems, we introduce a generalization of maximum likelihood estimation of the posterior of Gaussian processes for both regression and binary classification tasks. Here, we exploit neural networks to learn the feature space, and a set of inducing point in the feature space to reduce the computational cost of the covariance matrix. We report experiments on real data showing that our approach outperforms state-of-the-art approaches, and can be effectively combined with complex neural network architectures. Finally, our framework can be seen as a way to directly outputs uncertainty from neural networks.}

\end{abstract}

\section{Introduction}

Gaussian processes are powerful and attractive machine learning models, in particular in situations where uncertainty estimation is critical for performance, such as for medical diagnosis~\cite{dusenberry_analyzing_2020}.

Whereas the original Gaussian process formulation was limited in terms of scalability, there has been significant progress in scalable solutions with \cite{Quinonero-Candela_Rasmussen_2005} providing an early unified framework based on inducing points as a representative proxy of the training data. The framework by \cite{Quinonero-Candela_Rasmussen_2005} has also been extended to variational settings \cite{titsias_variational_2009,wilson_stochastic_2016}, further enabling a probabilistic basis for reasoning about the number of inducing points \cite{Uhrenholt_Charvet_Jensen_2021}. In terms of computational scalability, methods for leveraging the available computational resources have recently been considered \cite{Nguyen_Filippone_Michiardi_2019,Wang_Pleiss_Gardner_Tyree_Weinberger_Wilson_2019}, with \cite{Chen_Zheng} also providing insights into the theoretical underpinnings for gradient descent based solutions in correlated settings (as in the case for GPs).

Common for the inducing points-based approaches to scalability is that the inducing points live in the same space as the training points (see e.g.~\cite{snelson2006sparse,titsias_variational_2009,hensman2013gaussian,damianou2013deep}). However, learning inducing points in the input space can be challenging for complex domains (e.g.\ over graphs),  domains with high dimensionality (e.g.\ images), or domains with varying cardinality (e.g.\ text or points clouds) \cite{lee2019set,aitchison2021deep}.



In this paper, we propose inducing Gaussian process networks (\FW) as a simple and scalable framework for jointly learning the inducing points and the (deep) kernel \cite{Wilson_Hu_Salakhutdinov_Xing_2016}. Key to the framework is that the inducing points live in feature space rather than in the input space. By defining the inducing points in the feature space, we are able to represent the data distribution with a simple base kernel (such as the RBF and dot-product kernel), relying on the expressiveness of the learned features for capturing complex interactions. 

For learning \FWs, we rely on a maximum likelihood-based learning objective that is optimized using mini-batch gradient descent \cite{Chen_Zheng}. This setup allows the method to scale to large data sets as demonstrated in the experimental results. Furthermore, by only having the inducing points defined in feature space, we can seamlessly employ standard gradient-based techniques for learning the inducing points (even when the inputs space is defined over complex discrete/hybrid objects) without the practical difficulties sometimes encountered when learning deep neural network structures.


We evaluate the performance of the proposed framework on several well-known data sets and show significant improvements compared to state of the art methods. We provide a qualitative analysis of the framework using a two-class version of the MNIST dataset. This is complemented by a more detailed quantitative analysis using the full MNIST and CIFAR10 data sets. Lastly, to demonstrate the versatility of the framework, we also provide sentiment analysis results for both a text-based and a graph-based dataset derived from the IMDB movie review dataset.

\hide{Novelty and some implementation details:
\begin{itemize}
    \item Provide an uncertainty estimation for neural networks;
    \item GP are non parametric models, meaning that they depend on all the points in the training set. Set a small number of inducing points drastically reduces the complexity;
    \item We jointly learn the inducing points and the kernel;
    \item We apply the framework to any type of instance space (images, graphs, point cloud, text, and son on). Domains typically not explored by GP people.
    \item \alessandro{\textbf{DEPRECATED}. Superseded by the exact inverse matrix.}We approximate the inverse of $K_{II}$ by using the Netwon-Shulz method~\cite{higham2008} during training.
\end{itemize}}

\section{The inducing Gaussian process networks (\FW) framework}
\label{sec:framework}
We start by considering regression problems, defined over an input space $\mathcal{X}$
and a label space of observations $\mathbb{R}$, modeled by a Gaussian process:  
\begin{equation}\label{eq:gp}
    f \sim \mathcal{GP}(0, k(\cdot, \cdot)), \quad y = f(x) + \epsilon, \quad x \in \mathcal{X}, y \in \mathbb{R},
\end{equation}
where $k(\cdot,\cdot):\mathcal{X}\times \mathcal{X}\to \mathbb{R}$ denotes the kernel describing the prior covariance, and $\epsilon \sim \mathcal{N}(0, \sigma_{\epsilon}^2)$ is the noise associated with the observations. We assume access to a set of data points $\mathcal{D} = \{(\bmx_i,y_i)\}^n_{i=1}$ generated from the model in  Equation~\ref{eq:gp}, and we seek to learn the parameters that define $k$ and $\sigma_{\epsilon}$ in order to predict outputs for new points $\bmx_*$ in $\mathcal X$. In what follows we shall use $X$ and $\bmy$ to denote $(\bmx_1, \ldots ,\bmx_n)\T$ and $(y_1,\ldots y_n)\T$, respectively.

Firstly, we propose to embed the input points using a neural network $\bmg_{\bmtheta_g}:\mathcal{X}\to \mathbb{R}^d$ parameterized by $\bmtheta_g$.
Secondly, for modeling the kernel function $k$, we introduce a set of $m$ inducing points $Z = (\bmz_1, \ldots, \bmz_m)\T$, $\bmz_i \in \mathbb{R}^d$, together with a (linear) pseudo-label function $r_{\bmtheta_r}:\mathbb{R}^d \to \mathbb{R}$ parameterized by $\bmtheta_r$. We will use $\bmr = (r_{\bmtheta_r}(\bmz_1),\ldots, r_{\bmtheta_r}(\bmz_m))\T$ to denote the evaluation of $r$ on $Z$, where $\bmr$  will play a r\^{o}le similar to that of inducing variables \cite{Quinonero-Candela_Rasmussen_2005}. In the remainder of this paper, we will sometimes drop the parameter subscripts $\bmtheta_g$ and $\bmtheta_r$ from $\bmg$ and $\bmr$ for ease of representation. An illustration of the model and the relationship between the training data and the inducing points can be see in Figure~\ref{fig:GP-illustration}.
\hide{
Firstly, we propose to embed the input points using a neural network $\bmg_{\bmtheta_g}:\mathcal{X}\to \mathbb{R}^d$ parameterized by $\bmtheta_g$.
Secondly, we exploit a set of $m$ inducing points $I = {(\bmz_i, r_{\bmtheta_I}(\bmz_i))}_{i=1}^m$ for modeling the kernel function $k$, where
$\bmz_i \in \mathbb{R}^d$ and $r_{\bmtheta_I}:\mathbb{R}^d \to \mathbb{R}$ is a (linear) function parameterized by $\bmtheta_I$; $r_{\bmtheta_I}(\bmz_i)$ will play the r\^{o}le of inducing variables. In the remainder of this paper, we will sometimes drop the parameter subscripts $\bmtheta_g$ and $\bmtheta_I$ from $\bmg$ and $\bmr$ for ease of representation. An illustration of the model can be see in Figure~\ref{fig:GP-illustration}.
}
\begin{center}
\begin{figure}[htbp]
    \includegraphics[width=\columnwidth]{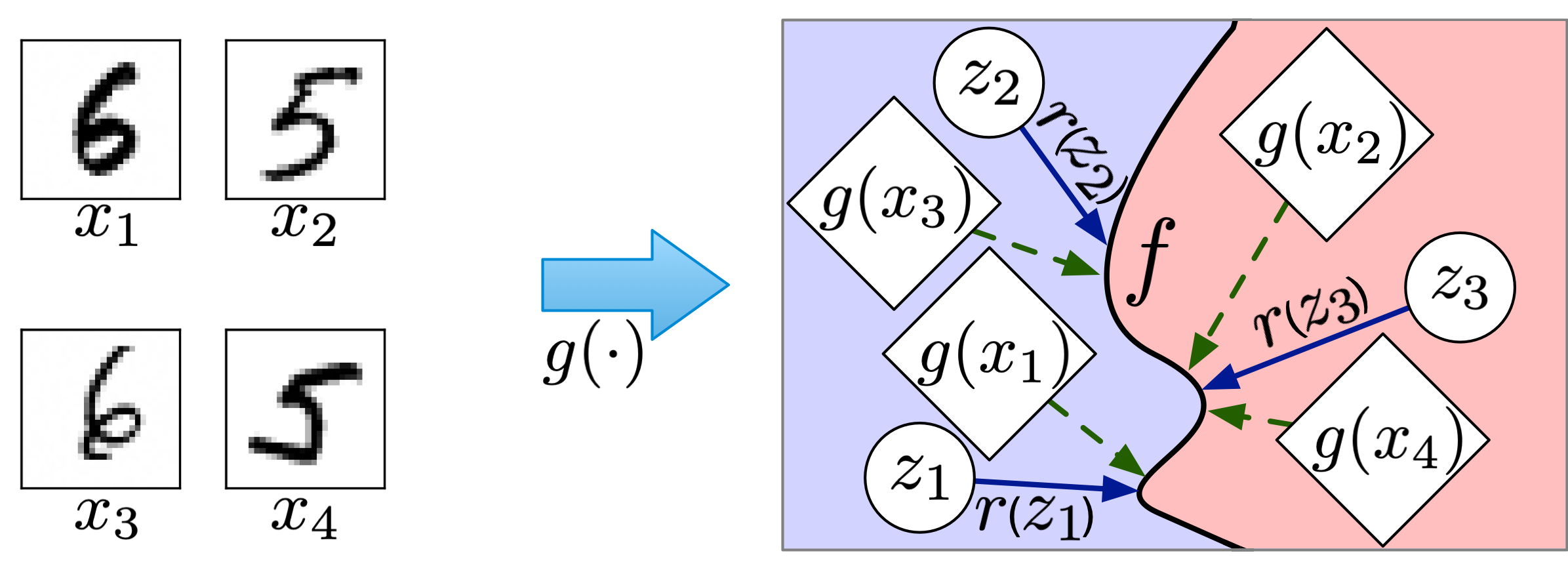}
    \caption{To the left four MNIST digits are embedded in the feature space by the neural network $\bmg$. To the right, the feature space where both features and inducing points exists. The observations associated with the inducing points $\bmz_1$, $\bmz_2$, and $\bmz_3$ are given by the pseudo-label function $r$, while the predictions associated with $\bmg(\bmx_1)$, $\bmg(\bmx_2)$, $\bmg(\bmx_3)$, and $\bmg(\bmx_4)$ are estimated using the  Gaussian process posterior.}
    \label{fig:GP-illustration}
\end{figure}    
\end{center}

We finally define 
$k:\mathbb{R}^d \times \mathbb{R}^d \to \mathbb{R}$ as the kernel between pairs of vectors in $\mathbb{R}^d$. In particular, we denote with
\begin{eqnarray*}\nonumber
    &(K_{XX})_{ij} = k(\bmg(\bmx_i), \bmg(\bmx_j)) \
     &(K_{ZX})_{ij}  = k(\bmz_i, \bmg(\bmx_j))  \\\nonumber
    &(K_{XZ})_{ij} = k(\bmg(\bmx_i), \bmz_j)) \
     &(K_{ZZ})_{ij} = k(\bmz_i, \bmz_j),
\end{eqnarray*}
the four matrices corresponding to the kernel evaluations for all pairs of observed inputs and inducing points. 

The joint distribution over the function values at the observed inputs and the inducing points is now given as
\begin{equation*}
\left [ 
\begin{matrix}
\bm{y} \\
\bm{r}
\end{matrix}
\right ] \sim \mathcal{N} \left( \bm{0}, \left [ 
\begin{matrix}
\quad K_{XX} + \sigma^2_{\epsilon}I  & K_{XZ} \\
K_{ZX} & K_{ZZ}
\end{matrix}
\right ]\right ).
\end{equation*}

For a given data set $\calD$, our goal is to jointly learn $\bmtheta = \{ Z, \bmtheta_g, \bmtheta_r,  \sigma_{\epsilon} \}$ by considering the marginal likelihood
\begin{equation}\label{eq:posterior}
p(\bmy | X, \bmtheta) = \mathcal{N}(\hat{\bmy}, K_{X|Z}),
\end{equation}
where $\hat{\bmy}$ is the predictive mean
\begin{equation}\label{eq:gp:cond}
    \hat{\bmy} = K_{XZ} K_{ZZ}^{-1} \bmr_{\bmtheta_r}
\end{equation}
and $K_{X|Z}$ is the posterior kernel given the inducing points, i.e.
\begin{equation}
    K_{X|Z} = (K_{XX}+\sigma^2_{\epsilon}I) - K_{XZ} K_{ZZ}^{-1} K_{ZX}.
\end{equation}
Note that the kernel $K_{X|Z}$ is implicitly parameterized by $\bmtheta_g$ through the embedding function $\bmg_{\bmtheta_g}$.

As in \cite{Chen_Zheng}, we define our objective function in terms of the negative log-likelihood
of the posterior of Equation~\ref{eq:posterior}:
\begin{eqnarray}\label{eq:ill:loss}
 \ell(\bmtheta;\mathcal{D}) =  & - &\frac{1}{2}(\bmy - \hat{\bmy})^T K_{X|Z}^{-1}(\bmy - \hat{\bmy}) -\frac{1}{2}\log |K_{X|Z}| \nonumber \\
 & - & \frac{n}{2}\log(2\pi), 
\end{eqnarray}
which we can minimize wrt.\ the parameters $\bmtheta$ using mini-natch gradient descent with appropriate scaling factors for the gradients \cite{Chen_Zheng}. The gradient of $\ell(\bmtheta;\mathcal{D})$ wrt.\ $\bmtheta$ can be expressed as:
\begin{equation}\label{eq:grad:ill:loss}
\begin{split}
\nabla_{\bmtheta} \ \ell(\bmtheta;\mathcal{D}) &=   K_{X|Z}^{-1}(\bmy - \hat{\bmy})\nabla_{\bmtheta}\hat{\bmy} -\frac{1}{2} \ tr(K_{X|Z}^{-1} \nabla_{\bmtheta}K_{X|Z})\\
 &  + K_{X|Z}^{-T}(\bmy - \hat{\bmy})(\bmy - \hat{\bmy})^T K_{X|Z}^{-T}\nabla_{\bmtheta}K_{X|Z}^{-1}, 
\end{split}
\end{equation}%
where $tr(A)$ represents the trace of a matrix $A$. Note that in Equation~\ref{eq:grad:ill:loss} both $\hat{\bmy}$ and $K_{X|Z}$ depend on $\bmtheta$. 

Based on the chain rule of differentiation, we see that the gradient of Equation~\ref{eq:grad:ill:loss} with respect to the inducing points $Z$ does not depend on $\nabla \bmg_{\bmtheta_g}$, that is the gradients of $\bmg_{\bmtheta_g}$ have no impact on the updates of $Z$. 
\begin{prop}
The gradient $\nabla_Z \ell(\bmtheta;\mathcal{D})$ does not depend on $\nabla \bmg_{\bmtheta_g}$.
\end{prop}
This is a key advantage of \FWs as learning the inducing points is therefore not influenced by potentially complex embedding functions $\bmg$ and the entailed optimization difficulties. Additionally, as the inducing points (only) exists in $\mathbb{R}^d$, the underlying learning framework for the inducing points is indifferent to the structure of the input space $\calX$ and whether it is discrete or continuous or defined over, e.g., graphs or images. 

Once the \FW parameters $\bmtheta = \{Z, \bmtheta_g, \bmtheta_r, \sigma_{\epsilon}\}$ have been learned, we can find the predictive distribution for input points $X_*$ and inducing points $Z$, by first considering the joint distribution over the associated function values
\begin{equation*}
\left [ 
\begin{matrix}
\bm{r} \\
\bm{f}_*
\end{matrix}
\right ] \sim \mathcal{N} \left( \bm{0}, \left [ 
\begin{matrix}
K_{ZZ} \quad K_{ZX_*} \\
K_{X_*Z} \quad K_{X_*X_*}
\end{matrix}
\right ]\right ),
\end{equation*}
which in turn gives the predictive distribution 
\begin{equation}\label{equ:GP_predictive}
p(\bmf_* | X_*, \bmtheta) = \calN \left( \hat{\bmf_*},  K_{X_* | Z}\right),
\end{equation}where
\begin{eqnarray*}
\hat{\bmf_*}  & = &  K_{X_*Z}K_{ZZ}^{-1}\bm{r} \\ 
K_{X_* | Z} & = & K_{X_*X_*}-K_{X_*Z}K_{ZZ}^{-1}K_{ZX_*}.
\end{eqnarray*}

\hide{
Note that if the inducing points and $r$ are known, for a given a set of point $X =\{ \bmx_i \ | \ (\bmx_i, y_i) \in \mathcal{D} \}$, we can condition the Gaussian process in order to get an estimate of the observations associated with each $\bmx_i$,
\begin{equation}\label{eq:gp:cond}
    \hat{\bmy} = K_{XI} K_{II}^{-1} r(I_z),
\end{equation}where $I_z=(\bmz_1,\ldots,\bmz_m)^T$. 

Our goal is to jointly learn $\{ I_z, \bmtheta_g, \bmtheta_I,  \sigma_{\epsilon} \}$ by considering the posterior 
\begin{equation}\label{eq:posterior}
p(\hat{\bmy} | X, I_z, \bmtheta_I) = \mathcal{N}(\hat{\bmy}, K_{X|I}),
\end{equation}where $K_{X|I}$ denotes the posterior kernel given the inducing points, i.e.
\begin{equation}
    K_{X|I} = K_{XX} - K_{XI} K_{II}^{-1} K_{IX}.
\end{equation}Note that $K_{X|I}$, provides an estimation of the uncertainty of the predictions of new data points, once the parameters are learned. Following the same derivation of~\cite{Chen_Zheng}, we can minimize the negative log-likelihood
of the posterior of Equation~\ref{eq:posterior}:
\begin{eqnarray}\label{eq:ill:loss}
 \ell(\theta;\mathcal{D}) =  & - &\frac{1}{2}(\bmy - \hat{\bmy})^T K_{X|I}(\bmy - \hat{\bmy}) -\frac{1}{2}\log |K_{X|I}| \nonumber \\
 & - & \frac{n}{2}\log(2\pi).
\end{eqnarray}
}

\subsection{\FWs for classification}\label{sec:gp:classification}

In this section, we extend the Gaussian process networks to classification. For ease of exposition, we only consider binary classification problems, but the general method can straightforwardly be extended to a multi-valued setting. For the experiment results, (see Section~\ref{sec:exp:mnist:cifar10}) we therefore use a one-vs-all approach to solve multi-class classification tasks. We assume that the label space is given by $\mathcal{Y} = \{0, 1\}$ and that we have access to a data set $\mathcal{D} = \{ (\bmx_i,y_i) \}_{i=1}^n$, where the underlying data generating process is defined by a latent function $f(\bmx)$ with a GP prior
\[
    f \sim \mathcal{GP}(0, k(\cdot, \cdot))
\]
and $y | \bm{x} \sim \Phi (f(\bm{x}))$, where $\Phi (\cdot)$ is the cumulative Gaussian function; see \cite{RasmussenW06}.

For inference, we calculate the posterior distribution $p(\bm{y}\given X_*, \bmtheta)$ given by
$$
p(\bm{y} \given X_*, \bmtheta) = \int p(y \given \bm{f}_*) p(\bmf_*\given X_*,\bmtheta)d\bm{f}_* \,.
$$
Assuming that $\bmf_*\sim \mathcal{N}(\mu, \sigma^2)$ for some $\mu$ and $\sigma$ we get \cite[Section 3.9]{RasmussenW06}:
$$
p(y_* \given \bm{x}_*, \bmtheta) = \Phi(\alpha) \text{ with } \alpha=\frac{\mu}{\sqrt{1+\sigma^2}},
$$
where $\mu$ and $\sigma^2$ can be found from the predictive distribution in Equation~\ref{equ:GP_predictive}.
\hide{
For the calculation of $p(\bm{f}_*\given X_*,\bmetheta)$, we have that:
\begin{equation*}
\left [ 
\begin{matrix}
\bm{r} \\
\bm{f}_*
\end{matrix}
\right ] \sim \mathcal{N} \left( \bm{0}, \left [ 
\begin{matrix}
K_{II} \quad K_{I*} \\
K_{*I} \quad K_{**}
\end{matrix}
\right ]\right ),
\end{equation*}
which gives the conditional distribution

\begin{equation}
\label{equ:pred}    
\bm{f}_* \given X_*, I_z, \bm{r} \sim \mathcal{N}\left( K_{*I}K_{II}^{-1}\bm{r}, K_{**}-K_{*I}K_{II}^{-1}K_{I*}\right ),
\end{equation}
where $\bmr_{\bmtheta_I} = (r_{\bmtheta_I}(\bmz_i))_{i=1:m}\T$
}

For learning the \FW parameters $\bmtheta$ we perform maximum likelihood estimation wrt.\ 
\begin{equation}\nonumber
\ell(\bmtheta; \calD) = \log p(\bm{y} \given X, \bmtheta) = \log \int p(\bm{y} \given \bm{f}) p(\bm{f}\given X,\bmtheta)d\bm{f}.
\end{equation}
In order to evaluate the integral, we rely on a Laplace approximation centered around
$$
\hat{\bm{f}}=\arg\max_{\bm{f}} \log p(\bmy | \bmf) + \log p(\bmf | X, \bmtheta),
$$ 
(found using Newton's method) yielding the log-likelihood approximation
\begin{equation}\nonumber
\ell(\bmtheta;\calD) \approx \log p(\bmy | \hat{\bmf}) + \log p(\hat{\bmf} | X, \bmtheta) - \frac{1}{2}\log(|A|) + c,
\end{equation}where $c$ represent the accumulated constant terms and 
\begin{equation}\nonumber
A={-\nabla\nabla (\log p(\bmy | \bmf) + \log p(\bmf | X, \bmtheta))}_{|\bm{f}=\hat{\bm{f}}}.
\end{equation}
The form of $p(\hat{\bmf} | X, \bmtheta)$ is given in Equation~\ref{eq:posterior} and $\nabla\nabla (\log p(\bmy | \hat{\bmf})$ can be found in, e.g., \cite{RasmussenW06}. 

Optimizing this log-likelihood approximation can now straightforwardly be achieved by interleaving mini-batch gradient descent and Newton's approximation. As for regression, we here also have that the gradient of (the Laplace approximation of) $\ell(\bmtheta;\calD)$ wrt.\ the inducing points $Z$ does not depend on $\nabla\bmtheta_g$ of the embedding function $\bmg$. Hence, \FWs for classification enjoy the same properties when learning the inducing points as \FWs for regression. Additional details can be found in the supplementary material.

\subsection{Computational complexity}
The main aspect related to the computational complexity of the \FW framework concerns the kernel computation. If $n_z$ is the number of inducing points and $b$ is the mini-batch size (recall that for training we use mini-batch gradient descent), then the complexity of computing the kernel grows linearly with $n_z$ , i.e. $\mathcal{O}(n_z \cdot b)$. Clearly, the choice of the number of inducing points has a strong impact on the computational complexity. The training process also consists of matrix inversion operations. However, as \FW can run on GPUs, matrix inversions can be efficiently computed~\cite{sharma2013fast}.


\section{Related works}
\label{sec:related_work}

The combination of kernels and neural networks has previously been explored, most notably in the context of deep kernel learning \cite{Wilson_Hu_Salakhutdinov_Xing_2016}. Using a a deep neural network architecture, \cite{Wilson_Hu_Salakhutdinov_Xing_2016} transform the input vectors into feature space based on which a base kernel is applied (here the RBF kernel and the spectral mixture base kernel \cite{Wilson_Adams_2013} are used). The kernel parameters and the neural network weights are jointly learned by maximizing the marginal likelihood of the Gaussian process, but relying on a pre-training of the underlying deep neural network architecture. This work has been subsequently extended into a variational setting \cite{wilson_stochastic_2016}, also providing support for multi-class classification.

A key difference between our approach and \cite{Wilson_Hu_Salakhutdinov_Xing_2016,wilson_stochastic_2016} is in our choice of inducing points. In \cite{Wilson_Hu_Salakhutdinov_Xing_2016}, the inducing points are placed on a regular multi-dimensional lattice based on which the deep kernel is evaluated. In \FW the inducing points are (only) defined in feature space, where they are treated as parameters and learned jointly together with the neural network weights and (any) kernel parameters. Furthermore, as shown in Section~\ref{sec:experiments}, our learning scheme is end-to-end and does not rely on pre-training of networks as in \cite{Wilson_Hu_Salakhutdinov_Xing_2016}. Finally, \FW relies on mini-batch gradient descent \cite{Chen_Zheng}, thus avoiding the GP-KISS kernel approximation of \cite{Wilson_Hu_Salakhutdinov_Xing_2016}. 
Several other related works exploit inducing points. For instance, \cite{titsias_variational_2009,hensman2013gaussian,damianou2013deep} propose to maximize a lower bound of the exact marginal likelihood to learn the inducing points in the input space, in contrast to our \FW framework where the inducing points are defined in feature space.
Closely related to \FW is~\cite{snelson2006sparse}, where inducing points are jointly learned with kernel parameters using gradient descent. However, in contrast to \FW, the continuous optimization of $Z$ proposed by~\cite{snelson2006sparse} is considerably more simplified as the inducing points are learned in the input space.
Similarly to the proposed framework, \cite{aitchison2021deep} exploit inducing Gram matrix. The Gram matrix are used to sample inducing points in the feature space, and subsequent GP predictions are conditionally sampled on the inducing points in the features space. This process, in contrast to \FWs, relies on a doubly-stochastic variational inference process. 


The papers cited above are representatives of GP methods that are methodologically related to the proposed \FW framework. Not all of the cited methods are, however, in line with state of the art in terms of, e.g., accuracy results, hence the experimental section below also includes descriptions and comparisons of other GP methods, which forms the basis for the empirical evaluation and analysis.

\section{Experiments}
\label{sec:experiments}
\begin{table*}[!htp]
  \caption{Comparison of root mean square error (RMSE) of different GPs on the benchmark datasets. We report the mean and standard error of RMSE. The best results are in bold (lower is better). For query and borehole datasets, EGP is not able to fit due to memory constraints.}
  \label{tab:rmse:gp:benchmarks}
  \centering
  \begin{small}
    \begin{sc}
      \begin{tabular}{cccccccc}
        \toprule
        Dataset & Size & D & \FW (Ours) & sgGP & EGP & SGPR & SVGP \\
        \midrule
        Levy & 10,000 & 4 & \tr[best]{0.178}{0.011} & \tr{0.265}{0.003} & \tr{0.312}{0.003} & \tr{0.564}{0.010} & \tr{0.582}{0.013} \\
        Griewank & 10,000 & 6 & \tr[best]{0.053}{0.004} & \tr{0.071}{0.000} & \tr{0.185}{0.073} & \tr{0.132}{0.003} & \tr{0.093}{0.005} \\
        Borehole & 1,000,000 & 8 & \tr[best]{0.003}{0.001} & \tr{0.172}{0.000} & --- & \tr{0.176}{0.000} & \tr{0.173}{0.000} \\
        \midrule
        Protein & 45,730 & 9 & \tr[best]{0.642}{0.007} & \tr{0.663}{0.006} & \tr{0.694}{0.004} & \tr{0.715}{0.003} & \tr{0.676}{0.004} \\
        PM2.5 & 41,757 & 15  & \tr{0.314}{0.008} & \tr[best]{0.287}{0.002} & \tr[best]{0.286}{0.003} & \tr{0.638}{0.005} & \tr{0.540}{0.010} \\
        Energy & 19,735 & 27 & \tr[best]{0.738}{0.015} & \tr{0.786}{0.001} & \tr{0.802}{0.067} & \tr{0.843}{0.006} & \tr{0.795}{0.005} \\
        Bike-hour & 17,379 & 15 & \tr[best]{0.006}{0.001} & \tr{0.221}{0.002}& \tr{0.228}{0.002}&\tr{0.276}{0.004}& \tr{0.250}{0.010} \\
        Query & 100,000 & 4 & \tr{0.077}{0.004} & \tr[best]{0.053}{0.000} & --- & \tr{0.058}{0.002}& \tr{0.061}{0.000} \\
        \bottomrule
      \end{tabular}
    \end{sc}
  \end{small}
\end{table*}
In this section, we present and discuss experimental results 
showing the potential of the \FW framework on both regression
and classification tasks. Both settings are aimed 
at showing the ability of \FWs to jointly learn
the inducing points and the model parameters. In the
experiments, we also compare \FWs against other state of the art methods that either learn the kernel parameters or the inducing points. The variety of the experiments
aim to provide insight into the properties of the framework and show the potential of \FWs for diverse real-world datasets, outperforming state-of-the-art methods in the majority of the cases.
\subsection{Regression tasks}\label{sec:exp:regression}
Following the setup described in~\cite{Chen_Zheng}, we compared our approach to PCG-based exact GP (EGP)~\cite{wang2019exact}, sparse GP regression (SGPR)~\cite{titsias_variational_2009}, stochastic variational GP (SVGP)~\cite{hensman2013gaussian}, and sgGP~\cite{Chen_Zheng} on several simulated and real regression benchmark datasets.  Briefly, EGP leverages GPU parallelization and conjugate gradients to compute the exact covariance matrix 
on large datasets, SGPR uses a variational formulation for sparse approximations that jointly
infers the inducing inputs and the kernel hyper-parameters by maximizing a lower bound
of the true log marginal likelihood, SVGP variationally decomposes 
Gaussian processes to depend on a set of globally relevant inducing variables, and sgGP exploits
stochastic gradient descent on the Gaussian process likelihood to learn the kernel hyper-parameters.

The sizes and feature dimensions for each dataset are reported in Table~\ref{tab:rmse:gp:benchmarks}. The simulated datasets (Levy, Griewank and Borehole) are from the Virtual Library of Simulation Experiments\footnote{\url{https://www.sfu.ca/~ssurjano/}}, while the real datasets (Protein, PM2.5, Energy, Bike-hour, and Query) are from the UCI repository\footnote{\url{http://archive.ics.uci.edu/ml}}. 

For all the methods, we report the root mean square error~\cite{Chen_Zheng} compared to the results in the original papers. More details about the other methods in the table can be found in the original publications referenced above. 

For each dataset, we repeated the experiments 10 times.
In each experiment, the dataset was randomly split into a 60\% training set and a 40\% test set. Furthermore, the training set was normalized to have 0 mean and 1 standard deviation, and the test set was scaled accordingly. For all the datasets, we considered a simple neural network
with 3 stacked dense layers with 128 units and ReLU activations. On top of those we added a final feature layer with 64 units. We used 512 inducing points having the same dimension as the feature layer (i.e., 64) and an RBF kernel with $\gamma=1.0$. The model weights are updated with the Adam optimizer for 500 epochs on mini-batches of size 128. 

Table~\ref{tab:rmse:gp:benchmarks} summarizes the results, where for each dataset we report the average and standard deviation of the root mean squared error on the test set over the 10 experimental repetitions. For Levy, Borehole, and Bike-hour datasets, our approach outperforms by a large margin the other methods. While for PM2.5 and Query datasets our approach performed slightly worse than the others, we still obtained comparable results. Overall, our method remained stable as the standard deviation over 10 experiments is small and comparable to sgGP.

\subsection{Toy-MNIST}\label{sec:toy:mnist}
Toy-MNIST is the same as MNIST, but limited to digits belonging to class 5 or 6. The aim of this experiments is to show the ability of \FWs to classify relatively
ambiguous digits, and provide insights into the inducing points and the uncertainty associated with the predictions. In this case, the training and test sets consist of 11,339 and 1,850 images, respectively. The image labels associated with classes 5 and 6 are replaced with 1 and -1, respectively, and we assume the exact same model specification as in Section~\ref{sec:exp:regression}, but now only using 64 inducing points.

By repeating the experiment 10 times we obtain a test set accuracy of $99.25\pm0.09\%$. For each of the 64 inducing points, we retrieve the closest image in the training set based on the chosen kernel in feature space, after which we refine
each of the images (relative to the distance to the closest inducing point) using gradient descent with respect to the input image. The 64 digit images corresponding to the inducing points are shown in Figure~\ref{fig:mnist:ip}.
\begin{figure}
    \centering
    \includegraphics[width=0.98\linewidth]{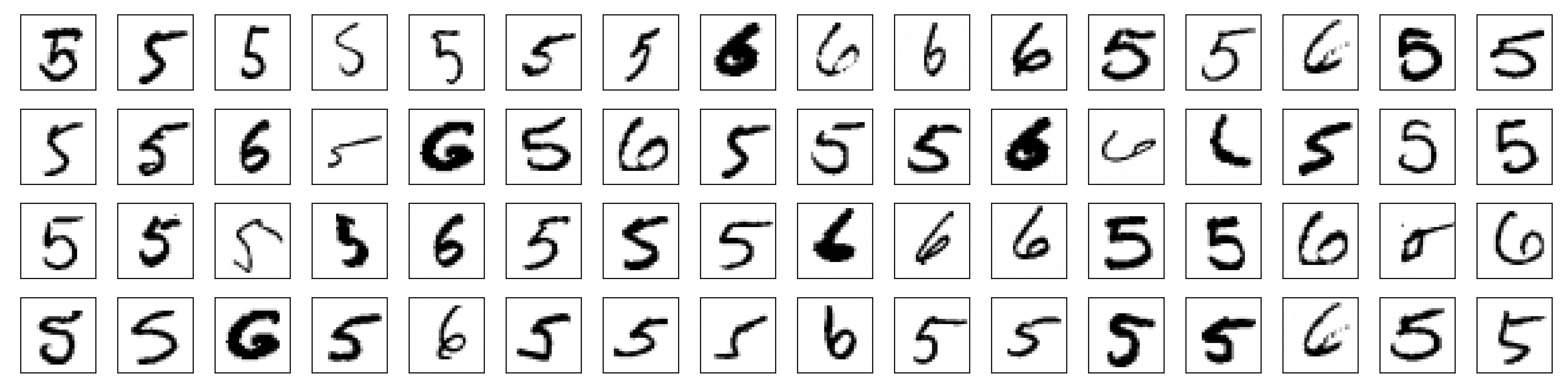}
    \caption{The 64 pseudo images corresponding to the 64 inducing points for the Toy-MNIST experiment.}
    \label{fig:mnist:ip}
\end{figure}
As with standard Gaussian processes, \FWs retains the feature of uncertainty estimation. This is illustrated in Figure~\ref{fig:mnist:top:errors}, which depicts a scatter plot of the images in the test set, given in terms of the expected value of the predictions (x-axis) and the variance of the predictions (y-axis). The two distinct orange and blue clusters contain the digits corresponding to classes 6 and 5. The figure also illustrates some of the extreme predictions: the two digits with means far from 0.5 and low variance correspond to clearly distinguishable digits, while the digit with a mean close to 0.5 and relatively high variance corresponds to a more ambiguous image.
\begin{figure}
    \centering
    \includegraphics[width=0.98\linewidth]{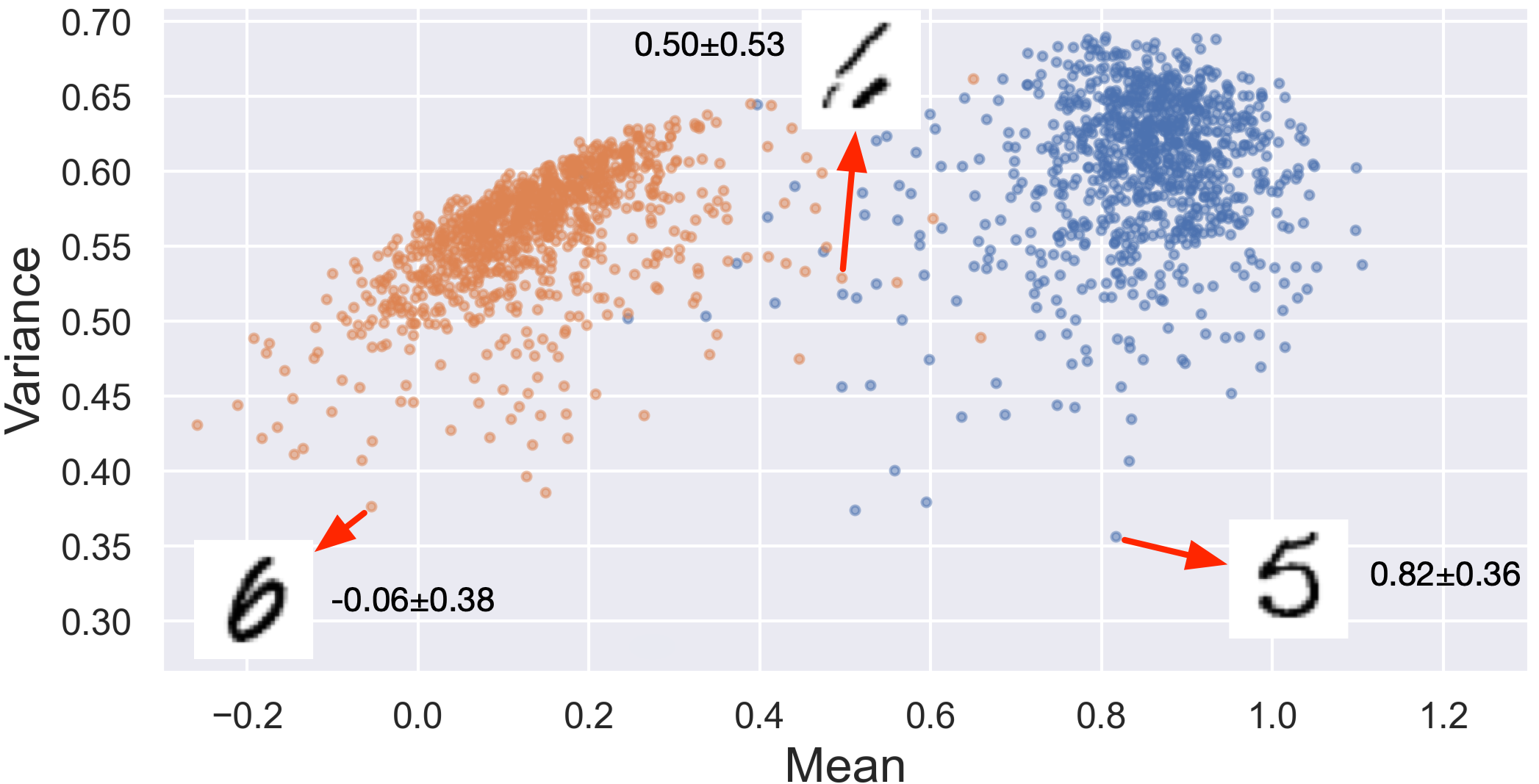}
    \caption{Mean (x-axis) against variance (y-axis) plot for the Toy-MNIST dataset. Orange  and blue dots correspond to images labeled with 6 and 5, respectively. Images associated with low variances and high mean (in absolute value) correspond to clear digits. On the other hand, images with high variance and mean close to the decision boundary, i.e. 0.5, correspond to ambiguous digits.}
    \label{fig:mnist:top:errors}
\end{figure}

\subsection{MNIST and CIFAR10}\label{sec:exp:mnist:cifar10}
In this section we evaluate \FW on computer vision classification tasks. 
Specifically, here we focus on MNIST and CIFAR10. We compare our approach against
deep kernel processes (DIWP)~\cite{aitchison2021deep} as well as Deep Gaussian Processes (DGP)~\cite{damianou2013deep} and NNGP~\cite{blundell_weight_2015}. Briefly, DIWP
defines deep kernel processes where positive
definite Gram matrices are progressively transformed by 
nonlinear kernel functions and by sampling from (inverse) Wishart distributions. DGP uses graphical models to nest layers of Gaussian processes. NNGP adopts a back-propagation-compatible algorithm for learning a probability distribution on the weights of a neural network to estimate the uncertainty of the model. 



For the sake of simplicity (see Section~\ref{sec:gp:classification}), we decompose with a one-vs-all approach each multiclass task into 10 binary tasks. We again used the same model as in Section~\ref{sec:exp:regression} (this model is comparable to the other competitor models in terms of number of parameters), but with 32 and 16 inducing points for MNIST and CIFAR10, respectively.
For each of the datasets, we repeated the experiments 10 times and report the average accuracy and standard deviation in Table~\ref{tab:gp:classification}. We observe that our method outperform all the others, while keeping the standard deviation comparably small.

\begin{table}[!ht]
  \caption{Test set accuracy comparison of MNSIT and CIFAR10 for \FW approach against DGP, NNGP, and DIWP.}
  \label{tab:gp:classification}
  \setlength\tabcolsep{3.0pt}
  \centering
  \begin{small}
    \begin{sc}
      \begin{tabular}{ccccc}
        \toprule
        Dataset & \FW (Ours) & DGP & NNGP & DIWP  \\
        \midrule
        MNIST & \tr[best]{98.00}{0.00} & \tr{96.5}{0.1} & \tr{96.5}{0.0} & \tr{97.7}{0.0} \\
        CIFAR-10 & \tr[best]{51.03}{0.2} & \tr{46.8}{0.1} & \tr{47.4}{0.1} & \tr{50.5}{0.1}  \\
        \bottomrule
      \end{tabular}
    \end{sc}
  \end{small}
\end{table}
\begin{figure*}[ht!]
    \centering
    \includegraphics[width=0.98\linewidth]{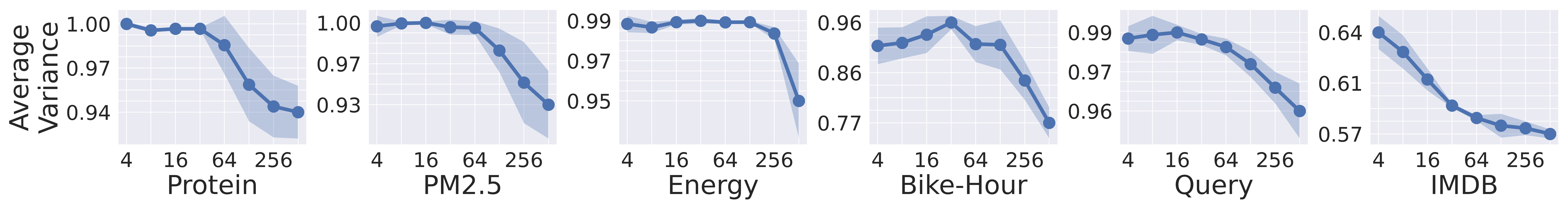}
    \caption{Average variance estimated on 10 runs for \textsc{Protein}, \textsc{PM2.5}, \textsc{Energy}, \textsc{Bike-Hour}, \textsc{Query}, and \textsc{IMBD-Text}. The average variance constantly decreases for all the datasets with number of inducing points.}
    \label{fig:ip:all}
\end{figure*}
\subsection{Structured datasets}\label{sec:structure:datasets}
In this section, we investigate the potential of our model applied to datasets with more complex structure. Specifically, we consider two classification tasks on text and graphs, respectively. The datasets for both tasks are derived from IMDB:
\begin{itemize}
    \item \textsc{IMDB-Text}~\cite{maas2011learning} consists of 25,000 training reviews and 25,000 test reviews. Positive and negative labels are balanced within the training and test sets. We preprocess the data by keeping the top 20,000 words and limiting the length of words in a review to 200.
    \item \textsc{IMDB-Graph} is a social network dataset contained in the collection described in~\cite{yanardag2015deep}. It consists of constructed genre-specific collaboration networks where nodes represent actresses/actors who are connected by an edge if they have appeared together in a movie of a given genre. Collaboration networks are generated for the genres Action and Romance for this dataset. The data then consists of the ego-graphs for all actresses/actors in all genre networks, and the task is to identify the genre from which an ego-graph has been extracted. It contains 1,000 graphs with labels balanced among the two classes.
\end{itemize}
For both experiments, we aim to show that our approach can jointly handle complex data and complex models for extracting features, hence the aim is not to compare against state-of-the-art methods. In particular, the two \FW instantiations for the two datsets differ only in the embedding functions/feature extractors being used. For both datasets, we also provide baseline accuracies obtained by two classifiers sharing the architectures of the feature extractors in the \FWs.  

For \textsc{IMDB-Text}, we used 128 inducing points and trained a transformer-based model \cite{vaswani2017attention} to learn features from the reviews. We repeated the experiments 10 times for the baseline and our approach obtaining accuracies of $84.35\% \pm 0.10$ and  $84.43\% \pm 0.96$, respectively. Similarly to Section~\ref{sec:toy:mnist}, we can also inspect the learned inducing points. For example, included below are the two training set examples (positive and negative) that are closets to two randomly chosen learned inducing points.\footnote{More training set review correspondences to inducing points are reported in the Supplementary material.}
\begin{itemize}
    \item \textit{I was literally preparing to hate this movie so believe me when I say this film is worth seeing ... my score 7 out of 10.}
    \item \textit{I can't believe that so many are comparing this movie to Argento's ... If you're looking for a good horror movie look elsewhere.}
\end{itemize}

For the \textsc{IMDB-Graph}, we follow the setup described by~\cite{tibo2020learning} and represent each graph as a set of neighborhoods. We performed a 10 times 10 fold cross-validation for both the baseline and \FW. We ran 100 epochs of the Adam optimizer with learning rate 0.001 on mini-batches of size 32. We obtained accuracies equal to $72.49\%\pm 0.60$ and $73.40\%\pm 0.63$, for the baseline and 
\FW, respectively. Without any particular effort, in terms of fine tuning model structure and parameters, our method outperforms most of the graph neural networks reported in~\cite{Nguyen2019UGT}.

\section{Ablation on number of inducing points}
Here we study the impact of the number of inducing points for the regression and classification tasks defined in the previous sections. The variance of Gaussian processes estimations decrease with the number of training points~\cite{RasmussenW06}. In our case, the training points are replaced by inducing points but we retain similar result. For the sake of completeness, we slightly reformulate the proposition and provide a proof in the supplementary material.
\begin{prop}
The variance of a test point $(\bmx_*, f_*)$ given a set of inducing points $Z$ can never increase by the inclusion of an additional inducing point $\bmz \not \in Z$. 
\end{prop}
When training \FWs with different number of inducing points, there is no guarantees that the same subset of inducing are learned during training. However, we still observe a consistent decrease in variance when the number of inducing points increases. For the regression tasks, we report here the results for all the real datasets described in Section~\ref{sec:exp:regression}. For the classification task we focus on \textsc{IMDB-Text} described in Section~\ref{sec:structure:datasets}. In all the cases, we repeated the experiments 10 times varying the number of inducing points (4,8,16,32,64,128,256). Figure~\ref{fig:ip:all} depicts the average variance with error bars for the different datasets. From the figure we see that the average variance consistently decreases with the number of inducing points for all datasets. For some of the datasets with only a few inducing points, the error bars are lower. In these cases, the models were always unable to learn due to the limited number of inducing points. The outputs of the models are therefore constant for those cases. 

\section{Conclusions and future work}
While Gaussian processes are powerful machine learning methods that can estimate uncertainty, they remain intractable on large datasets. Several works have addressed this problem using inducing points but, to the best of our knowledge, these methods remain limited to 
comparatively simple datasets. In this paper, we introduce a framework for learning Gaussian processes
for large and complex datasets by combining inducing points in feature space with neural network kernel approximations.  We empirically showed the ability of our method on standard machine learning benchmarks as well as on structured (graph and text) datasets, with the proposed method outperforming other state-of-the-art approaches.

The flexibility of our \FW framework enables the combination of Gaussian process training with complex deep learning architectures. We believe that our method provides a basis for further research in calibration of uncertainty estimates in deep neural network models. Furthermore, as part of future work, we also plan to investigate how to position \FW within a fully variational setting~\cite{titsias_variational_2009}, exploring the hierarchical structure of the model, possibly in the context of deep kernel processes \cite{aitchison2021deep}.


\subsection*{Acknowledgments}
The work was partly done in connection with the CLAIRE project, Controlling water in an urban environment, supported by VILLUM FONDEN, research grant 34262.

\clearpage

\section*{Appendix: About the variance of inducing points}
\begin{prop}
The variance of a test point $(\bmx_*, f_*)$ given a set of inducing points $Z$ can never increase by the inclusion of an additional inducing point $\bmz \not \in Z$. 
\end{prop}
\begin{proof}
Let $S$ and $R$ be two sets containing inducing points.
Let $K_{f_* , \bms | \bmr}$ be the variance associated with $f_*$ conditioned only on the inducing point in $R$. Let $K_{f_* |  \bms, \bmr}$ be the variance associated with $f_*$  conditioned only on the inducing point in $R \cup S$. If $K_{X_*X_*}'$ is the top-left scalar in $K_{f_* , \bms | \bmr}$ representing the variance of $f_*$, then we prove that $K_{f_* |  \bms, \bmr} \le K_{X_*X_*}'$.
Let $K$ be the covariance matrix of the Gaussian process
having the following block structure:
\begin{equation*}
K = \left [ 
\begin{matrix}
K_{X_*X_*} & K_{X_*R} & K_{X_*S} \\
K_{RX_*} & K_{RR} & K_{RS} \\
K_{SX_*} & K_{SR} & K_{SS}
\end{matrix}
\right ],
\end{equation*}where the subscripts represents the two sets (among the observation $\{\bmx_* \}$ and inducing points in $R$ and $S$) used to evaluate the kernel matrix. We compute now $K_{f_* , \bms | \bmr}$ as
\begin{align}
K_{f_* , \bms | \bmr} & =  \left [ 
\begin{matrix}
K_{X_*X_*} & K_{X_*S} \\
K_{SX_*} & K_{SS}
\end{matrix}
\right ] - 
\left [ 
\begin{matrix}
K_{X_*R} \\
K_{SR}
\end{matrix}
\right ] K_{RR}^{-1}
\left [ 
\begin{matrix}
K_{RX} K_{RS}
\end{matrix}
\right ] \nonumber \\
& =  \left [ 
\begin{matrix}
K_{X_*X_*}' & K_{X_*S}' \\
K_{SX_*}' & K_{SS}'
\end{matrix}
\right ].
\end{align}Note that $K_{X_*X_*}' \le K_{X_*X_*}$. Finally, 
\begin{equation*}
K_{f_*  | \bms ,  \bmr} = K_{X_*X_*}' - K_{X_*S}'K_{SS}'^{-1}K_{SX_*}' \le K_{X_*X_*}'
\end{equation*}which concludes the proof.
\end{proof}

\subsection*{\FWs for classification}\label{sec:gp:classification}
Here we provide additional details about the full derivation (see Section 2.1 in the main paper) of \FWs for classification. For ease of exposition, we only consider binary classification problems, but the general method can straightforwardly be extended to a multi-valued setting. For the experiment results, (see Section 4.3 in the main paper) we therefore use a one-vs-all approach to solve multiclass classification tasks. We assume that the label space is given by $\mathcal{Y} = \{0, 1\}$ and that we have access to a data set $\mathcal{D} = \{ (\bmx_i,y_i) \}_{i=1}^n$, where the underlying data generating process is defined by a latent function $f(\bmx)$ with a GP prior
\[
    f \sim \mathcal{GP}(0, k(\cdot, \cdot))
\]
and $y | \bm{x} \sim \Phi (f(\bm{x}))$, where $\Phi (\cdot)$ is the cumulative Gaussian function; see \cite{RasmussenW06}.

For inference, we calculate the posterior distribution $p(\bm{y}\given X_*, \bmtheta)$ given by
$$
p(\bm{y} \given X_*, \bmtheta) = \int p(y \given \bm{f}_*) p(\bmf_*\given X_*,\bmtheta)d\bm{f}_* \,.
$$
Assuming that $\bmf_*\sim \mathcal{N}(\mu, \sigma^2)$ for some $\mu$ and $\sigma$ we get \cite[Section 3.9]{RasmussenW06}:
$$
p(y_* \given \bm{x}_*, \bmtheta) = \Phi(\alpha) \text{ with } \alpha=\frac{\mu}{\sqrt{1+\sigma^2}},
$$
where $\mu$ and $\sigma^2$ can be found from the predictive distribution in Equation 8 (in the main paper).


For learning the \FW parameters $\bmtheta$ we perform maximum likelihood estimation, but using the Laplace approximation to derive an approximation of the marginal likelihood. Firstly, the likelihood is defined as
\begin{equation}\label{eq:gp:binary} 
p(\bm{y} \given X, \bmtheta) = \int p(\bm{y} \given \bm{f}) p(\bm{f}\given X,\bmtheta)d\bm{f}.
\end{equation}Denoting by 
$$
\hat{\bm{f}}=\arg\max_{\bm{f}} \log p(\bmy | \bmf) + \log p(\bmf | X, \bmtheta)
$$ and employing the Laplace approximation in Equation~\ref{eq:gp:binary}, our goal is to maximize the following log-likelihood approximation
\begin{equation}\label{eq:gp:binary:s2}
\ell(\bmtheta;\calD) \approx \log p(\bmy | \hat{\bmf}) + \log p(\hat{\bmf} | X, \bmtheta) - \frac{1}{2}\log(|A|) + c,
\end{equation}where $c$ represent the accumulated constant terms and 
\begin{equation}\label{eq:gp:binary:hessian}
A={-\nabla\nabla (\log p(\bmy | \bmf) + \log p(\bmf | X, \bmtheta))}_{|\bm{f}=\hat{\bm{f}}}.
\end{equation}The form of $p(\hat{\bmf} | X, \bmtheta)$ is given in Equation 3 (in the main paper) and using the shorthand notation
\begin{eqnarray*}
\bma &=&K_{XZ}K_{ZZ}^{-1}\bm{r}_{\bmtheta_r} \\
K &=& K_{XX}-K_{XZ}K_{ZZ}^{-1}K_{ZX},
\end{eqnarray*}
the log-likelihood in Equation~\ref{eq:gp:binary:s2} can be expressed as 
\begin{eqnarray}\label{eq:gp:binary:s3}
\ell(\bmtheta;\calD) & \approx & \log p(\bmy | \hat{\bmf}) -\frac{1}{2}\log(|A|) - \frac{1}{2}\log(|K|) \nonumber \\
& - &\frac{1}{2}(\hat{\bmf} - \bma)^T K^{-1} (\hat{\bmf} - \bma)+ c.
\end{eqnarray}The term $p(\bmy | \bmf)$ in  Equation~\ref{eq:gp:binary:hessian} factorizes, hence $W = \nabla\nabla \log p(\bmy | \bmf)$ is diagonal \cite[Page 43]{RasmussenW06} with
\begin{eqnarray}
W_{ii} &=& \nabla \nabla \log p(y_i\given f_i) = \nabla \nabla \Phi (y_i\cdot f_i) \nonumber \\
&=& -\frac{\mathcal(f_i)^2}{\Phi(y_i\cdot f_i)^2} - \frac{y_i\cdot f_i \cdot \mathcal(f_i)}{\Phi(y_i\cdot f_i)}. \nonumber
\end{eqnarray}From Equation~\ref{eq:gp:binary:s3}, $\nabla\nabla\log p(\bmf | X, \bmtheta) = K^{-1}$, therefore $A=-W-K^{-1}$. We finally approximating the log-likelihood (omitting constant terms) as
\begin{eqnarray*}\label{eq:gp:binary:s4}
\ell(\bmtheta;\calD) &\approx& \log p(\bmy | \hat{\bmf}) -\frac{1}{2}\log(|-W-K^{-1}||K|)\\
&-& \frac{1}{2}(\hat{\bmf} - \bma)^T K^{-1} (\hat{\bmf} - \bma).
\end{eqnarray*}Lastly, we find $\hat{\bmf}$ using Newton's approximation for $T$ steps. Starting with $\hat{\bmf}^{0}=\bm{0}$, for $1 \le t \le T$, we have
$$
\hat{\bmf}^{t+1} = (K^{-1} + W)^{-1}(W\hat{\bmf}^{t} + \nabla\log p(\bmy|\hat{\bmf}^{t})),
$$where ${\nabla\log p(y_i|f_i)} = y_i\calN(f_i)/\Phi(y_i f_i)$.

\section*{Appendix: IMDB-Text Inducing Points}
We report the training set reviews close to the inducing points for the IMBD-Text experiment presented in Section 4.4 for one of the ten runs. Recall that we use 128 inducing points. In this case, we obtained 60 unique inducing points corresponding to 32 positive and 28 negative training set reviews. The inducing points are all different to each other but there are collisions when choosing the closests points in the training set. Note that the inducing points are balanced among the two classes. For the sake of space, for each review we report the initial and last 25 words.
\begin{enumerate}
\item \textbf{positive} - there have been several films about zorro some even made in europe e g alain delon this role has also been played by outstanding actors  ...  good performance of hadley as zorro he was quick smart used well his whip and sword and his voice was the best for any zorro
\item \textbf{negative} - i love the frequently masters of horror series horror fans live in a constant lack of projects like this and the similar project with gave  ...  up has to have a payoff that exceeds build up not the other way around storytelling math 101 br br end of spoilers big oops
\item \textbf{positive} - i rented the film i don't think it got a theatrical release here out expecting the worse the previews made the film look awful i  ...  be proclaimed 'the worst film ever i recommend this film for anybody interested in the show a flawed but innovative and interesting piece of film
\item \textbf{positive} - one of the best of the fred astaire and rogers films great music by irving berlin solid support from randolph scott harriet nelson lucille ball  ...  jazzy and it's a great song br br fun all the way although i got tired of we joined the navy after the third time
\item \textbf{positive} - forbidden planet rates as landmark in science fiction carefully staying within hard aspects of the genre science not fantasy nerds will love it while still  ...  the edge destroying its creator just as it did thousands of centuries earlier to the krell br br maybe the krell had teenage daughters too
\item \textbf{negative} - the bad news is it's still really dreadful i gave it a 2 because occasionally some of this slapstick parody actually seems funny br br  ...  this one and get this dvd back to the video store on time you'll really hate yourself if you have to pay a late fee
\item \textbf{positive} - nice character development in a pretty cool milieu being a male i'm probably not qualified to totally understand it but they do a nice job  ...  but within this world it needed to happen good acting all around with something positive taking place in the lives of some pretty good people
\item \textbf{negative} - mercifully there's no video of this wannabe western that a stay afloat vehicle for big frank at a time when his career was floundering the  ...  you up late and having a bout of insomnia but if you can sit through it you've more than most of my movie buff friends
\item \textbf{negative} - bela lugosi is an evil who sends brides poisoned on their wedding day steals the body in his fake ambulance hearse and takes it home  ...  in a discount store 2 for £1 which i think is a pretty accurate anyone paying more for this would be out of their mind
\item \textbf{positive} - its no surprise that busey later developed a in his this film is also a poor decision but one i enjoyed fully the first 5  ...  wet myself some of best work by far rent or buy it today my vote is a perfect 10 on the poo meter that is
\item \textbf{positive} - before sky i saw diane tender performance in this otherwise of a movie campers are invited to the camp of their youth and experience it  ...  comic acting turn by noted director sam raimi makes this a movie you can pull out again and again like looking up an old friend
\item \textbf{positive} - the line of course is from the lord's prayer thy will be done on earth as it is in heaven sweden especially its far north  ...  the ending is what you make of it i guess but it's not spoiling it to say daniel achieves what he set out to do
\item \textbf{positive} - my certainly is a fair looking woman this film is a lost gem a dead on satire mockumentary of the early 90's hip hop scene  ...  this regard i regard this movie like the 1000 islands of upstate new york it's a wonderful little secret you want to keep to yourself
\item \textbf{positive} - i rented this film from netflix for two reasons i was in the mood for what i thought would be a silly '50s sci fi  ...  also generally very good and the acting is much better than one might expect i was particularly impressed with reeves jeff corey and walter reed
\item \textbf{negative} - brilliant book with wonderful characterizations and insights into human nature particularly the nature of addiction which still resonate strongly today br br as for the  ...  normally excellent but inappropriately cast actors all in all a weak adaptation your three hours would be better spent reading or re reading the book
\item \textbf{negative} - it's boggles the mind how this movie was nominated for seven oscars and won one not because it's abysmal or because given the collective credentials  ...  director hungry to be recognized it could've been morphed to something better but what's left looks like a film nobody was really interested in making
\item \textbf{positive} - a comedy of funny proportions from the guys that brought you south park and most of the guys from this movie has utterly disgusting and  ...  turn the sport sour and its up ta coop ta fix it and along the way you will laugh alot that's all there is enjoy
\item \textbf{negative} - there was a bugs bunny cartoon titled baby buggy bunny that was exactly this plot baby faced robbed a bank and the money in the  ...  the bugs bunny dvd it's was much more original the first time 1954 plus you'll get a lot more classic bugs bunny cartoons to boot
\item \textbf{negative} - first off i really loved henry fool which puts me in a very small pool of movie goers parker posey is one of best actresses  ...  ride i'd be happy to spoil this movie for you but it's been done it's rotten the fool franchise is dead long live henry fool
\item \textbf{positive} - i imagine victorian literature slowly sinking into the of the increasingly distant past pulled down by the weight of its under skirts along comes television  ...  coarse have been made to modern tastes and without having felt preached to another bbc classic highly recommended this is how romantic literature should be
\item \textbf{negative} - i knew this movie wasn't going to be amazing but i thought i would give it a chance i am a fan of luke wilson  ...  the movie without people getting annoyed the movie had its moments but i'm glad i didn't spend 9 50 to see it in the theater
\item \textbf{negative} - i have read several good reviews that have defended and the various aspects of this film one thing i see over and over is annoyance  ...  of good and terrible acting i would recommend it for a cheap thrill but hardly a diamond in the rough that is micro budget horror
\item \textbf{negative} - this movie is like the thousand cat and mouse movies that preceded it the following may look like a spoiler but it really just describes  ...  exiting the theater from a hollywood movie and if you have ever felt that way too heed my warning stay miles away from this movie
\item \textbf{negative} - i wonder who how and more importantly why the decision to call richard attenborough to direct the most singular sensation to hit broadway in many  ...  well michael douglas was in it true i forgot i'm absolutely wrong and you are absolutely right nothing like a richard attenborough michael douglas musical
\item \textbf{positive} - this movie will go down down in history as one of the greats right along side of citizen kane casablanca and on the waterfront someone  ...  do yourself and your family a favor and buy it immediately i'm still holding out hope for a special edition dvd one of these days
\item \textbf{negative} - this is high grade cheese fare of b movie kung fu flicks bruce wannabe lee is played by bruce li i think of course let's  ...  flashback for a scene just shown 3 minutes ago they must've thought that only one with attention disorder could fully understand this film br br
\item \textbf{negative} - the only previous gordon film i had watched was the kiddie adventure the magic sword 1962 though i followed this soon after with empire of  ...  them then again this particular version is further sunk by the tacked on electronic score – which is wholly inappropriate and cheesy in the extreme
\item \textbf{positive} - hilarious evocative confusing brilliant film reminds me of or holy mountain lots of strange characters about and looking for what is it i laughed almost  ...  watch on screen or at his big slide show smart funny quirky and outrageously hot make more films write more books keep the nightmare alive
\item \textbf{positive} - the villian in this movie is one mean sob and he seems to enjoy what he is doing that is what i guess makes him  ...  guess you can make up your own mind about the true ending i'm left feeling that only one character should have survived at the end
\item \textbf{negative} - okay what the hell kind of trash have i been watching now the mountain has got to be one of the most incoherent and insane  ...  good heroine this is the type of european horror film that could have been legendary if only someone had bothered to write a structured screenplay
\item \textbf{positive} - i found this movie to be very good in all areas the acting was brilliant from all characters especially ms stone and character just gets  ...  audience which was misled by some faulty terrible reviews about the movie before it even started you won't regret it if you go see it
\item \textbf{negative} - a lot of horror fans seem to love scarecrows so i won't be very popular in saying that i found it to be rather boring  ...  involving killer scarecrows to my knowledge apart from dark night of the scarecrow which is much better i would recommend that over scarecrows any day
\item \textbf{negative} - david mamet's film debut has been hailed by many as a real thinking man's movie a movie that makes you question everybody and everything i  ...  unfulfilled and if you like me predicted ahead of time that margaret was going to be conned you will find this revelation just as unsatisfying
\item \textbf{positive} - while this was a better movie than 101 dalmations live action not animated version i think it still fell a little short of what disney  ...  as so many disney films are here's to hoping the third will be even better still because you know they probably want to make one
\item \textbf{positive} - i think i read this someplace joe johnston director of the film and also one of the guys who founded industrial light and magic for  ...  first homer jr did not like the idea but he warmed up to it after the movie poster paperback novel came out and took off
\item \textbf{positive} - stephen king movies are a funny thing with me i either really love them or i loathe them some of the productions such as desperation  ...  very watchable and enjoyable adaption br br for uk readers this production has most recently been shown on sci fi and sky thriller horror channels
\item \textbf{positive} - if you're researching ufo facts then this video is very important the of the video is the comments made by buzz he is without a  ...  in details should not detour your from viewing this video if nothing else it is interesting and i recommend you watch with an open mind
\item \textbf{positive} - rupert friend gives a performance as prince albert that lifts the young victoria to unexpected levels he is superb as we know queen victoria fell  ...  believe for a minute she was victoria no real sense of period it may no have been her fault but her prince deserved the crown
\item \textbf{positive} - seeing moonstruck after so many years is a reminder of how sweet and funny this film was when it first appeared who knew that cher  ...  used to be at its best entertainment with no social significance whatsoever if they'd only lost that's along the way it would have been perfect
\item \textbf{negative} - i first learned of the wendigo many years ago in one of alvin scary stories books according to that story the wendigo after calling your  ...  to count br br anyway avoid it patricia clarkson and erik per sullivan dewey on malcolm in the middle have done far better than this
\item \textbf{positive} - well maybe not immediately before the rodney king riots but even a few months before was timely enough my parents said that they saw it  ...  but either way grand canyon is a great movie it kevin kline as my favorite actor also starring mary mary louise parker and alfre woodard
\item \textbf{positive} - finally a movie where the audience is kept guessing until the end what will happen well we all kind of know that the lives of  ...  his drug and sex addictions and a father who finally discovers exactly what happened the day of the robbery this movie will get you thinking
\item \textbf{negative} - wow i don't even really remember that much about this movie except that it stunk br br the plot's basically a girl's parents neglect her  ...  you do see it don't expect much 1 out of 10 br br seriously if you want a pokemon movie rent pokemon the first movie
\item \textbf{positive} - the group of people are travelling to in an awful bus led by a drunk conductor and his dumb son who likes to drive with  ...  bigger than him in the end the movie takes one turn and the trip becomes nothing but a swan s song of a dying country
\item \textbf{negative} - yes this movie is a real thief it stole some shiny oscars from just because politicians wanted another war hero movie to boost the acceptance  ...  if we consider this title a reasonable piece of the u s wars are cool genre you surely have much better movies to choose from
\item \textbf{positive} - i was going through a list of oscar winners and was surprised to see that this film beat butch cassidy and the sundance kid for  ...  by hoffman to take this role otherwise he may have been typecast after the graduate anyway this considered an all time great for a reason
\item \textbf{positive} - inspirational tales about triumph of the human spirit are usually big turn offs for me the most surprising thing about men of honor is how  ...  doesn't disappoint he creates a darkly funny portrait director george jr set out to make an old style flick and comes up with a winner
\item \textbf{negative} - this movie was god awful from conception to execution the us needs to set up a star wars site in this remote country this is  ...  gymnast star in real life i would probably kick him in the face after a double with 2 1 2 twists in the layout position
\item \textbf{positive} - i wouldn't call we're back a story simply a kiddie version of jurassic park i found it more interesting than that like the former it  ...  kind i would actually say that john goodman doing voice here is sort of a precursor to his voice work in monsters inc worth seeing
\item \textbf{negative} - chinese ghost story iii is a totally superfluous sequel to two excellent fantasy films the film delivers the spell casting special effects that one can  ...  a little extra money out of a successful formula they won't be able to do it again the cash cow is now dead as a
\item \textbf{negative} - the direction had clearly stated that this film's idea and plot is totally original however as to those who have read comic we can clearly  ...  watching this thus making this movie getting what it shouldn't have it has became one of the best budget films in china for this year
\item \textbf{negative} - let me start by saying that i understand that invasion of the star creatures was meant to be a parody of the sci fi films  ...  the double feature dvd with invasion of the bee girls that movie is academy award winning stuff in comparison with invasion of the star creatures
\item \textbf{positive} - knowing when to end a movie is just as important as casting directing and acting and it's nice to see when a director script get  ...  a mansion br br this is a great independent production and one that wastes little time getting going and it won't waste your time either
\item \textbf{negative} - italian born has inherited from her deceased lover karl an ultra modern and isolated house in the middle of the woods it's winter and she  ...  around for it i think i'll give it the benefit of the doubt as it's definitely not what i was expecting from this indie film
\item \textbf{negative} - is a horror comedy that doesn't really have enough horror or comedy to qualify as one or the other it has one scene that is  ...  the movie that is weaker in general plot and spine because of production values that just shows you how uninteresting i found the look of
\item \textbf{positive} - there were a lot of 50's sci fi movies they were big draws for the drive in theaters a lot of them were crappy even  ...  worried about the invisible monster forbidden planet is a movie a sci fi fan can watch several times and find something new with each viewing
\item \textbf{negative} - curiously it is rene eyes and mouth not buddy the that emerge as the focal point of buddy a jim henson pictures production through francis  ...  thompson needs a good pick up shot she gives rene another extreme close up i wonder what the lipstick budget was on this picture from
\item \textbf{positive} - now this is what i'd call a good horror with occult supernatural undertones this nice low budget french movie caught my attention from the very  ...  very confusing towards the end but redeems itself by the time it's over br br i thought his was a very good movie 8 10
\item \textbf{negative} - and ethel buffs too will love her loud vocals as the wicked witch but this cartoon sequel to the wizard of oz is bereft of  ...  sure baby boomers will get a charge from it since it has been out of for so long as a curiosity item just fair from
\item \textbf{negative} - this is another one of those vs insects features a theme that was popular in the late 70's only you can't really call it horror  ...  after having seen ants lacking suspense action thrills shocks and creepiness the only thing you'll be left with after seeing ants is an annoying itch
\end{enumerate}

\bibliographystyle{named}

\end{document}